\documentclass[10pt,twocolumn,letterpaper]{article}
\usepackage{iccv}
\usepackage{times}
\usepackage{adjustbox}
\usepackage{amsmath}
\usepackage{amssymb}
\usepackage{amsthm}
\usepackage{apxproof}
\usepackage{booktabs}
\usepackage{courier}
\usepackage{cuted}
\usepackage{graphicx}
\usepackage{longtable}
\usepackage{multirow}
\usepackage{multirow}
\usepackage{todonotes}
\usepackage[pagebackref=true,breaklinks=true,letterpaper=true,colorlinks,bookmarks=false]{hyperref}
\usepackage{cleveref}
\usepackage{caption}
\usepackage{afterpage} 

\renewcommand{\appendixsectionformat}[2]{B} 

\newtheoremrep{lemma}{Lemma}
\newtheorem{definition}{Definition}
\newcommand{\sota}{state-of-the-art}

\newcommand{\thanksnote}[1]{{%
  \let\thempfn\relax
  \footnotetext[1]{#1}
}}

\newcommand{\methodname}{C3DPO}

\newcommand{\rdrstyle}{metallic}

\makeatletter
\renewcommand{\paragraph}{%
  \@startsection{paragraph}{4}%
  {\z@}{0.5em}{-1em}%
  {\normalfont\normalsize\bfseries}%
}
\makeatother

\makeatletter
\def\smallunderbrace#1{\mathop{\vtop{\m@th\ialign{##\crcr
   $\hfil\displaystyle{#1}\hfil$\crcr
   \noalign{\kern3\p@\nointerlineskip}%
   \tiny\upbracefill\crcr\noalign{\kern3\p@}}}}\limits}
\makeatother

\iccvfinalcopy
\def\iccvPaperID{17}
\ificcvfinal\fi

\newcommand{\figmargin}{-0.2cm}
\newcommand{\tabmargin}{-0.2cm}

\title{C3DPO: Canonical 3D Pose Networks for Non-Rigid Structure From Motion}

\begin{document}

\twocolumn[{%
\renewcommand\twocolumn[1][]{#1}%
\newpage
\null
\begin{center}
\vskip .375in
{\Large \bf
C3DPO: Canonical 3D Pose Networks for Non-Rigid Structure From Motion
\par}
\vspace*{24pt}
{
\large
\lineskip .5em
David Novotny$^{\ast}$~~~~%
Nikhila Ravi$^{\ast}$~~~~%
Benjamin Graham~~~~%
Natalia Neverova~~~~%
Andrea Vedaldi
\vspace{0.1cm} \\
{\tt\small \{dnovotny,nikhilar,benjamingraham,nneverova,vedaldi\}@fb.com}
\vspace{0.4cm} \\
{Facebook AI Research}
\vspace{0.2cm}
\par
}
\vspace{0.53cm}
\includegraphics[width=0.95\linewidth]{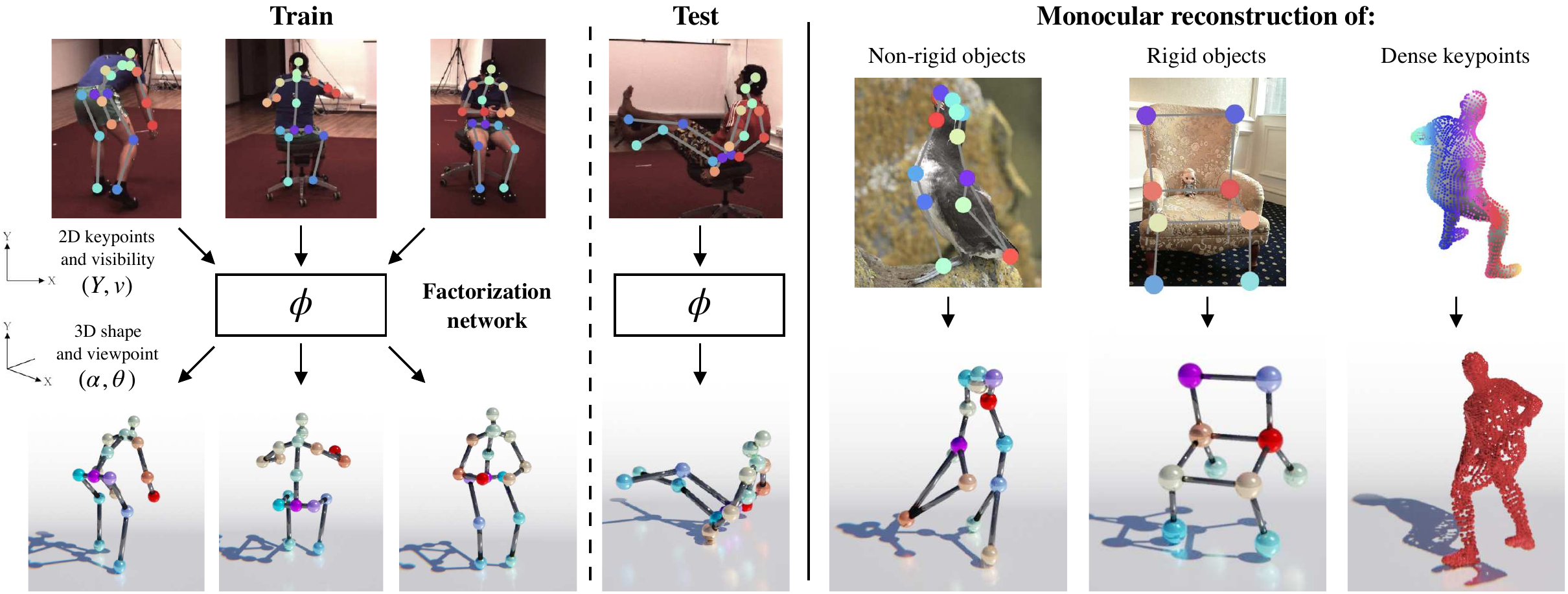}
\captionof{figure}{Our method learns a 3D model of a deformable object category from 2D keypoints in unconstrained images. It comprises a deep network that learns to factorize shape and viewpoint and, at test time, performs monocular reconstruction. \label{f:splash}} 
\end{center}
}
\vspace{0.3cm}
]

\begin{abstract} 
  \vspace{-0.3cm}
  We propose \methodname, a method for extracting 3D models of deformable objects from 2D keypoint annotations in unconstrained images. We do so by learning a deep network that reconstructs a 3D object from a single view at a time, accounting for partial occlusions, and explicitly factoring the effects of viewpoint changes and object deformations. In order to achieve this factorization, we introduce a novel regularization technique. We first show that the factorization is successful if, and only if, there exists a certain canonicalization function of the reconstructed shapes. Then, we learn the canonicalization function together with the reconstruction one, which constrains the result to be consistent. We demonstrate state-of-the-art reconstruction results for methods that do not use ground-truth 3D supervision for a number of benchmarks, including Up3D and PASCAL3D+.
Source code has been made available at
\url{https://github.com/facebookresearch/c3dpo_nrsfm}.
\end{abstract}
\afterpage{\thanksnote{$^{\ast}$Authors contributed equally.}}
\section{Introduction}\label{s:intro}

3D reconstruction of static scenes is mature, but the problem remains challenging when objects can deform due to articulation and intra-class variations.
In some cases, deformations can be avoided by capturing multiple simultaneous images of the object.
However, this requires expensive hardware comprising several imaging sensors and only provides instantaneous 3D reconstructions of the objects without modelling their deformations.
Extracting deformation models requires establishing correspondences between the instantaneous 3D reconstructions, which is often done by means of physical markers.
Modern systems such as the Panoptic Studio~\cite{joo2015panoptic} can align 3D reconstructions without markers, but require complex specialized hardware, making them unsuitable for use outside a specialized laboratory.

In this paper, we thus consider the problem of reconstructing and modelling 3D deformable objects given only unconstrained monocular views and keypoint annotations.
Traditionally, this problem has been regarded as a generalization of static scene reconstruction, and approached by extending Structure from Motion (SFM) techniques.
Due to their legacy, such Non-Rigid SFM (NR-SFM) methods have often focused on the geometric aspects of the problem, but the quality of the reconstructions also depends on the ability to \emph{model statistically} the object shapes and deformations.

We argue that modern deep learning techniques may be used in NR-SFM to capture much better statistical models of the data than the simple low-rank constraints employed in traditional approaches.
We thus propose a method that reconstructs the object in 3D while learning a deep network that models it.
This network is inspired by recent approaches~\cite{kudo2018unsupervised, kanazawa2018learning, pavlakos2018learning, drover20183dpose, kar2015category} that accurately lift 2D keypoints to 3D given a single view of the object.
The difference is that our network does not require 3D information for supervision, but is instead trained jointly with 3D reconstruction from 2D keypoints.

Our model, named \methodname, has two important innovations. First, it performs 3D reconstruction by \emph{factoring} the effects of viewpoint changes and object deformations. Hence, it reconstructs the 3D object in a canonical frame that registers the overall 3D rigid motion and leaves as residual variability only the motion ``internal'' to the object.

However, achieving this factorization correctly is non-trivial, as noted extensively in the NR-SFM literature~\cite{xiao2004dense}.
Our second innovation is a solution to this problem.
We observe that, if two 3D reconstructions overlap up to a rigid motion, they must coincide (since the reconstruction network should remove the effect of a rigid motion).
Hence, any class of 3D shapes equivalent up to a rigid motion must contain at most one canonical reconstruction.
If so, there exits a ``canonicalization'' function that maps elements in each equivalent class to this canonical reconstruction.
We exploit this fact by learning, together with the reconstruction network, a second network that performs this canonicalization, which regularizes the solution.

Empirically, we show that these innovations lead to a very effective and robust approach to non-rigid reconstruction and modelling of deformable 3D objects from unconstrained 2D keypoint data.
We compare \methodname~against several traditional NR-SFM baselines as well as other approaches that use deep learning \cite{kanazawa2018learning,kudo2018unsupervised}.
We test on a number of benchmarks, including Human3.6M, PASCAL3D+, and Synthetic Up3D, showing superior results for methods that make no use of ground-truth 3D information.

\section{Related work}\label{s:related}


There are several lines of work which address the problem of 3D shape and viewpoint recovery of a deforming object from 2D observations. This section covers relevant work in NR-SFM and recent deep-learning based methods.

\paragraph{NR-SFM.}

There are several solutions to the NR-SFM problem which can recover the viewpoint and 3D shape of a deforming object from 2D keypoints across multiple frames~\cite{akhter2009nonrigid, bregler2000recovering, akhter2011trajectory, dai2014simple}, the majority of which are based on Bregler's factorization framework \cite{bregler2000recovering}.
However the NR-SFM problem is severely under constrained as both the camera and 3D object are moving along with the object deforming. This poses a challenge in correctly factoring the viewpoint and shape~\cite{xiao2004dense}, and additional problems with missing values in the observations.
Priors about the shape and the camera motion are employed to improve conditioning of the problem, including the use of low-rank subspaces in the spatial domain ~\cite{agudo2018image, fragkiadaki2014grouping, dai2014simple, zhu2014complex}, temporal domain, for example, fitting 2D keypoint trajectories to a set of predefined DCT basis functions~\cite{akhter2009nonrigid, akhter2011trajectory}, spatio-temporal domain ~\cite{agudo2017dust, gotardo2011non, kumar2018scalable, kumar2017spatial}, multiple unions of low-rank subspaces~\cite{zhu2014complex, agudo2018deformable}, learning an overcomplete dictionary of basis shapes from 3D motion capture data and imposing an L1 penalty on basis coefficients~\cite{zhou2016sparseness, zhou2016sparse} and imposing Gaussian priors on the shape coefficients ~\cite{torresani2008nonrigid}. 

Many of these approaches however, as we have empirically verified, are not scalable and can only reliably reconstruct datasets of few thousands of images and hundreds of keypoints. Furthermore, many of them require keypoint correspondences for the \emph{same} instance from \emph{multiple} images from a monocular view or from multi-view cameras. 
Finally, in contrast to our method, using the listed approaches it is difficult or computationally expensive to reconstruct new test samples after training on a fixed collection of training shapes.

\begin{figure*}[ht!]
\centering\includegraphics[width=\textwidth,trim={0cm 4.3cm 0cm 0cm},clip]{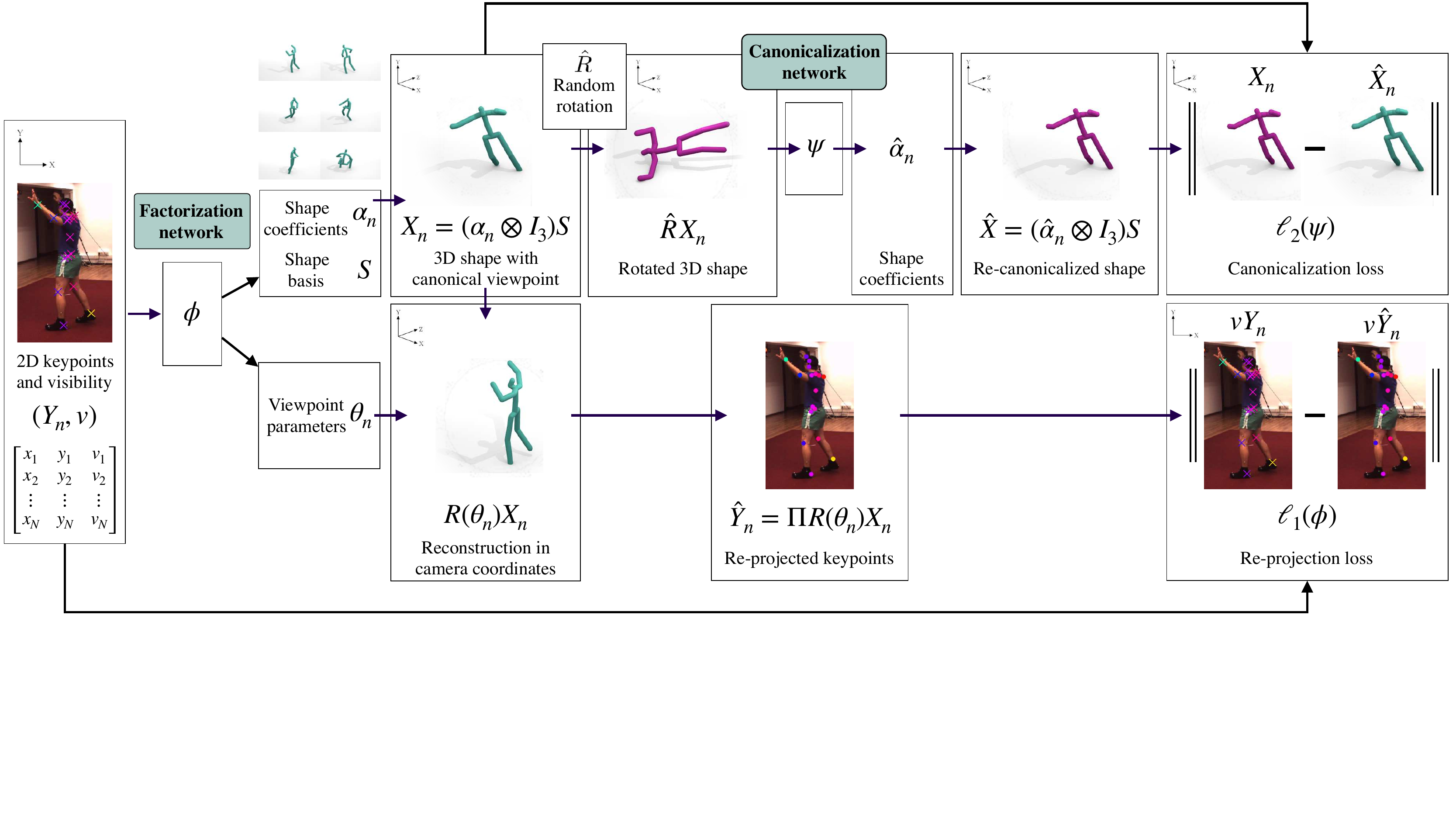}\vspace{-0.4cm}
\caption{\textbf{An overview of \methodname.} The lower branch learns monocular 3D reconstruction by minimizing the re-projection error $\ell_1$. The upper branch learns to factorize viewpoints and internal deformations by the means of the canonicalization loss.}\label{f:overview}
\vspace{-0.3cm}
\end{figure*}

\paragraph{Category specific 3D shapes.}
Also related are methods that reconstruct shapes of a visual object category, such as cars or birds. \cite{cashman2013shape} is an early work that learns a morphable model of dolphins from 2D keypoints and segmentation masks. Using similar supervision, Vicente \etal \cite{vicente2014reconstructing,carreira2015virtual} reconstruct the categories of PASCAL VOC. An important part of the pipeline is an initial SFM algorithm which returns a mean shape and camera matrices of each object category. Similarly, Kar \etal \cite{kar2015category} utilize an NR-SFM method for reconstructing categories from PASCAL3D+. \cite{novotny17learning} proposed the first purely image-driven method for single-view reconstruction of rigid object categories. Most recently, Kanazawa \etal \cite{kanazawa2018learning} train a deep network capable of learning both shape and texture of deformable objects. The commonality among the aforementioned methods is their reliance on the initial SFM/NR-SFM step which can often fail. Our method overcomes this problem by learning a monocular shape predictor in a single step without any additional, potentially unreliable, preprocessing steps. 

\paragraph{Weakly supervised 3D human pose estimation.}
Our approach is related to weakly supervised methods that lift 2D human skeleton keypoints to 3D given a single input view. Besides the fully supervised methods \cite{martinez2017simple,moreno20173d}, several works have explored multi-view supervision \cite{kocabas2019self,pavlakos2017harvesting,rhodin2018learning}, ordinal depth supervision \cite{pavlakos2018ordinal}, unpaired 2D-3D data \cite{pavlakos2018learning,tung2017adversarial,zhou2016sparseness,kanazawa2018end} or videos \cite{kanazawa2019learning} to alleviate the need for full 2D-3D annotations. While these auxiliary sources of supervision allow for compelling 3D predictions, in this work we use only inexpensive 2D keypoint labels.

Closer to our supervisory scheme, \cite{kudo2018unsupervised,drover20183dpose} recently proposed a method that rotates the 3D-lifted keypoints into new views and validates the resulting projections with an adversarial network that learns the distribution of plausible 2D poses.
However, both methods require \emph{all} keypoints to be visible in every frame. This restricts their use to `multi-view' datasets such as Human3.6M. In addition to the 2D keypoints, \cite{drover20183dpose} use the intrinsic camera parameters, and 3D ground truth data to generate new synthetic 2D views, which leads to substantially better quantitative results at the cost of a greater level of supervision.

To conclude, our contribution differs from prior work as it 1) recovers both 3D canonical shape and viewpoint using only 2D keypoints in a single image at test time, 2) uses a novel self-supervised constraint to correctly factorize 3D shape and viewpoint, 3) can handle occlusions and missing values in the observations, 4) works effectively across multiple object categories.


\section{Method}\label{s:method}

We start by summarizing some background facts about SFM and NR-SFM and then we introduce our method.

\subsection{Structure from motion}\label{s:sfm}

The input to \emph{structure from motion} (SFM) are tuples $y_n=(y_{n1},\dots,y_{nK}) \in \mathbb{R}^{2\times K}$ of 2D keypoints, representing $N$ \emph{views} $y_1,\dots,y_n$ of a rigid object.
The views are generated from a single tuple of 3D points $X =(X_1,\dots,X_K)\in \mathbb{R}^{3\times K}$, called the \emph{structure}, and $N$ \emph{rigid motions} $(R_n,T_n) \in SO(3) \times T(3)$.
The views, the structure, and the motions are related by equations
$
  y_{nk} = \Pi(R_n X_k + T_n)
$
where $\Pi:\mathbb{R}^3\rightarrow\mathbb{R}^2$ is the \emph{camera projection} function.
For simplicity of exposition we consider an orthographic camera.
In this case, the projection function is linear and given by matrix
$
\Pi = \begin{bmatrix} I_2 & 0 \end{bmatrix}
$
where $I_2 \in \mathbb{R}^{2\times 2}$ is the 2D identity matrix and the projection equation
$
 y_{nk} = \Pi R_n X_k + \Pi T_n
$
is also linear.
If all keypoints are visible, they can be centered together with the structure, eliminating the translation from this equation (details in the supplementary material). This yields the simplified system of equations
$
y_{nk} = M_n X_k,
$
where
$
M_n = \Pi R_n
$
are the camera view matrices, or \emph{viewpoints}.
The equations can be written in matrix form as
\begin{equation}\label{e:sfm}
\setlength\arraycolsep{2pt}
Y
=
\begin{bmatrix}
y_{11} & \dots & y_{1K} \\
\vdots & \ddots &  \vdots\\
y_{N1} & \dots & y_{NK} \\
\end{bmatrix}\!\!,
\ %
M
=
\begin{bmatrix}
M_1\\\vdots\\ M_N
\end{bmatrix}\!\!,
\smallunderbrace{Y}_{2N\times  K}
\!\!\!=\!
\smallunderbrace{M}_{2N\times 3} \smallunderbrace{X}_{3\times K}.
\end{equation}
Hence, SFM can be formulated as factoring the views $Y$ into viewpoints $M$ and structure $X$.
This factorization is not unique, resulting in a mild reconstruction ambiguity, as discussed in supplementary material.

\subsection{Non-rigid structure from motion}\label{s:nrsfm}

The \emph{non-rigid} SFM (NR-SFM) problem is similar to the SFM problem, except that the structure $X_n$ is allowed to deform from one view to the next.
Obtaining a non-trivial solution is only possible if such deformations are constrained in some manner.
The simplest constraint is a linear model $X_n = X(\alpha_n;S)$, expressing the structure $X_n$ as a small vector of view-specific \emph{pose} parameters $\alpha_{n} \in \mathbb{R}^D$ and a view-invariant \emph{shape basis} $S\in\mathbb{R}^{3D \times K}$:
\begin{equation}\label{e:gen}
X(\alpha_n;S) = (\alpha_n \otimes I_3) S
\end{equation}
where $\alpha_n$ is a row vector and $\otimes$ is the Kronecker product.
We can expand the equation for individual points as
$
X_{nk} = \sum_{d=1}^D \alpha_{nd} S_{dk}
$
where $S_{dk}\in\mathbb{R}^3$ is a shorthand for the subvector $S_{3d-2:3d,k}$.
We can also extend it to all points and poses as
$
X = (\alpha \otimes I_3) S \in \mathbb{R}^{3N \times K}
$
where $\alpha \in\mathbb{R}^{N\times D}$ encodes a pose per row.

Given multiple views of the points, the goal of NR-SFM is to recover the views, the poses, and the shape basis from observations
$
y_{nk} = \Pi(R_n \sum_{d=1}^D \alpha_{nd} S_{dk} + T_n).
$
As in SFM, for orthographic projection the translation can be removed from the equation by centering, and NR-SFM can be expressed as a multi-linear matrix factorization problem:
\begin{equation}\label{e:nrsfm}
 \smallunderbrace{Y}_{2N\times K} =
 \smallunderbrace{\bar M}_{2N\times 3N\text{\ (sparse)}}
 (\smallunderbrace{\alpha}_{N\times D} \!\!\!\otimes\ I_3)\smallunderbrace{S}_{3D\times K},
\end{equation}
where the $N$ camera view matrices are contained in the block-diagonal matrix
$
\bar M = \operatorname{diag}(M_1,\dots,M_N).
$
Like SFM, this factorization has mild ambiguities, discussed in the supplementary material.

\subsection{Monocular motion and structure estimation}

Once the shape basis $S$ is learned, model~\eqref{e:nrsfm} can be used to reconstruct viewpoint and pose given a single view $Y$ of the object, yielding  monocular reconstruction.
However, this still requires solving a matrix factorization problem.

For \methodname, we propose to instead \emph{learn} a mapping $\Phi$ that performs this factorization in a feed-forward manner, recovering the view matrix $M$ and the pose parameters $\alpha$ from the keypoints $Y$:
$$
\Phi :
\mathbb{R}^{2K} \times \{0,1\}^K
\rightarrow \mathbb{R}^D \times \mathbb{R}^{3},
\quad
(Y, v) \mapsto (\alpha,\theta).
$$
Here, $v$ is a (row) vector of boolean flags denoting whether a keypoint is visible in that particular view or not (if the keypoint is not visible, the flag as well as the spatial coordinates of the point are set to zero).
The function outputs the $D$ pose parameters $\alpha$ and the three parameters $\theta\in\mathbb{R}^3$ of the camera view matrix $M(\theta)= \Pi R(\theta)$, where the rotation is given by
$
R(\theta) = \operatorname{expm} [\theta]_\times
$,
$\operatorname{expm}$ is the matrix exponential and $[\cdot]_\times$ is the hat operator.

The benefit of using a learned mapping, besides speed, is the fact that it can embody prior information on the structure of the object which is not apparent in the linear model.
The mapping itself is learned by minimizing the \emph{re-projection loss} obtained by averaging the loss over visible keypoints:
\begin{equation}\label{e:basicloss}
  \ell_1(Y,v;\Phi,S)
  =
  \frac{1}{K}
  \sum_{k=1}^K
  v_k
  \cdot
  \| Y_k - M(\theta) (\alpha \otimes I_{3}) S_{:,k} \|_\epsilon,
\end{equation}
where $(\alpha, \theta) = \Phi(Y, v)$ and $\| z \|_\epsilon = (\sqrt{1+(\|z\|/\epsilon})^2-1) \epsilon$ is the pseudo-huber loss with soft threshold $\epsilon$\footnote{We set $\epsilon=0.01$ in all experiments.}.
Given a dataset $(Y,v)\in\mathcal{D}$ of views of an object category, the neural network $\Phi$ is trained by minimizing the empirical average of this loss.
This setup is illustrated in the bottom half of~\cref{f:overview}.

\subsection{Consistent factorization via canonicalization}\label{s:canonicalization}

\begin{figure}[t]
\newcommand{\eqimw}{1.3cm}

\newcommand{\eqbox}[1] {\scalebox{1}[-1]{\adjustbox{width=\eqimw,trim={.35\width} {.30\height} {.35\width}  {.30\height},clip}{\includegraphics[width=\linewidth]{#1}}}}
\newcommand{\eqboxx}[1]{\scalebox{1}[-1]{\adjustbox{width=\eqimw,trim={.3\width} {.05\height} {.3\width} {.47\height},clip}{\includegraphics[width=\linewidth]{#1}}}}

\small
\begin{center}
Input 2D keypoint locations $Y$: \\ \vspace{-0.3cm}
\eqbox{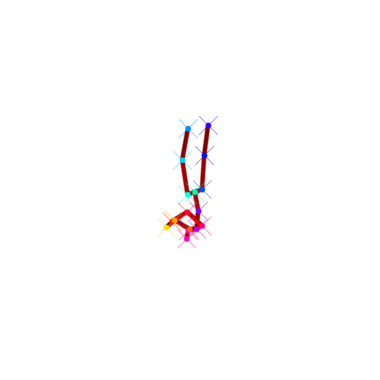}
\eqbox{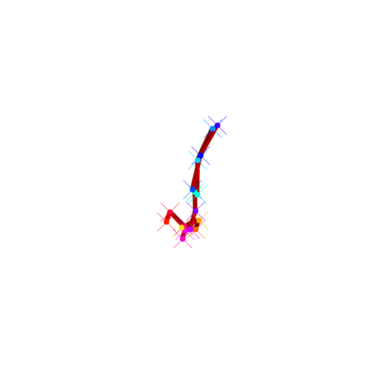}
\eqbox{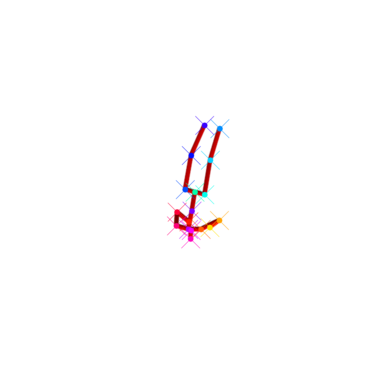}
\eqbox{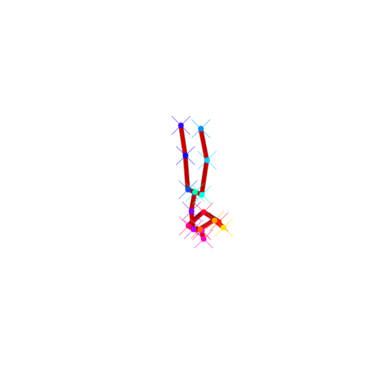}
\eqbox{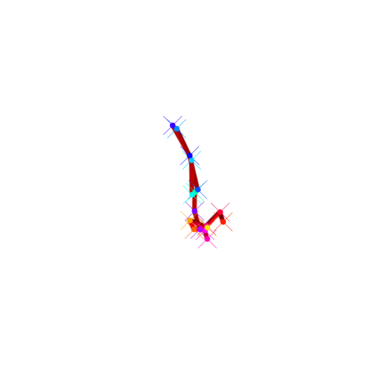}
\eqbox{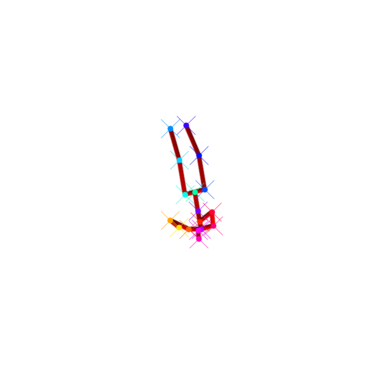}\\
\vspace{0.2cm}
Predicted canonical shape $X = \Phi(Y)$ trained \textbf{with} $\Psi$: \\ \vspace{-0.3cm}
\eqboxx{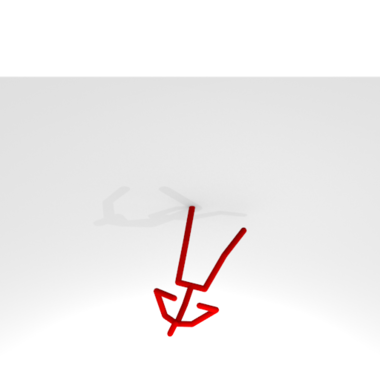}
\eqboxx{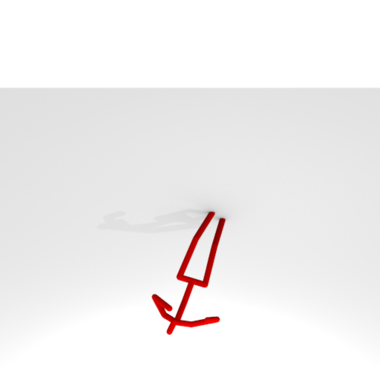}
\eqboxx{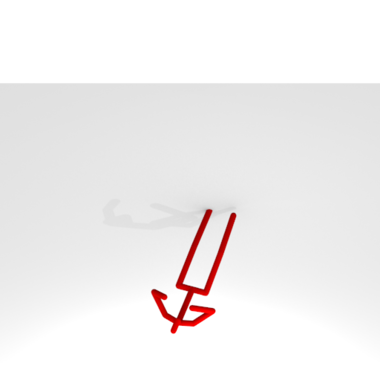}
\eqboxx{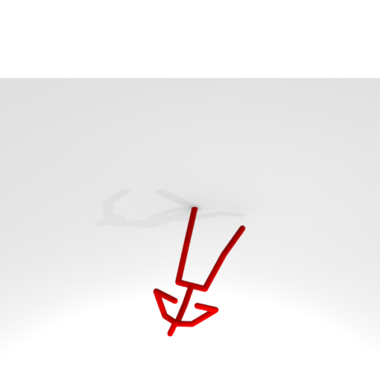}
\eqboxx{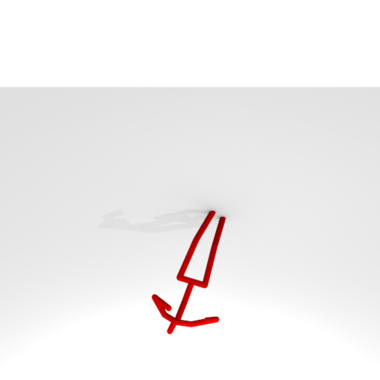}
\eqboxx{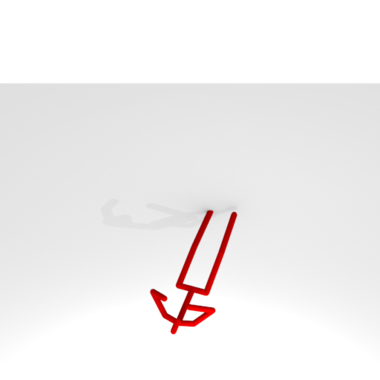}\\
\vspace{0.2cm}
Predicted canonical shape $X = \Phi(Y)$ trained \textbf{without} $\Psi$: \\ \vspace{-0.3cm}
\eqboxx{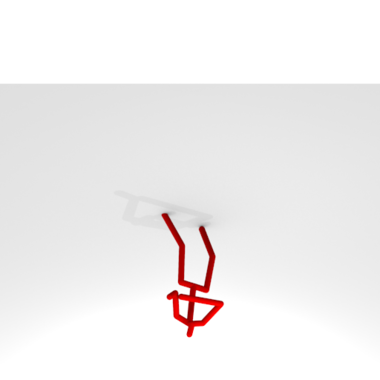}
\eqboxx{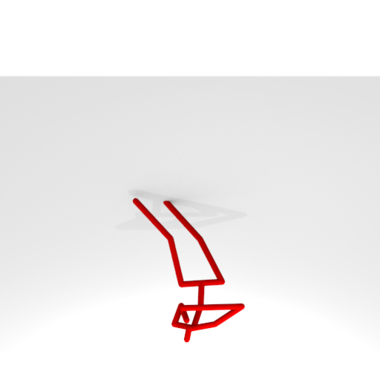}
\eqboxx{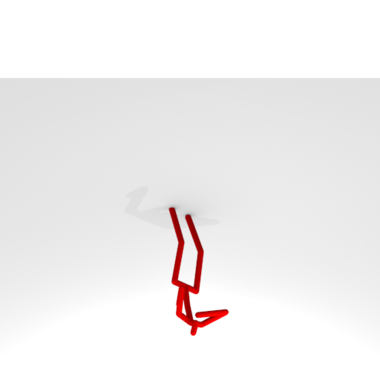}
\eqboxx{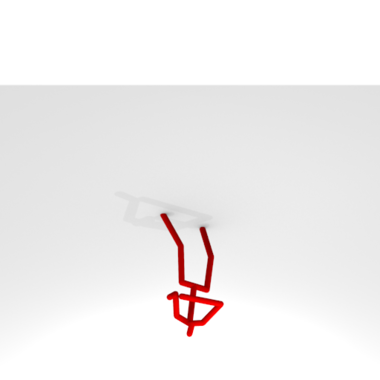}
\eqboxx{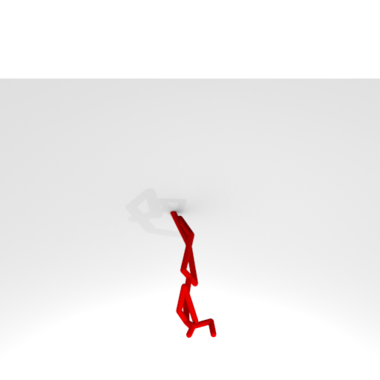}
\eqboxx{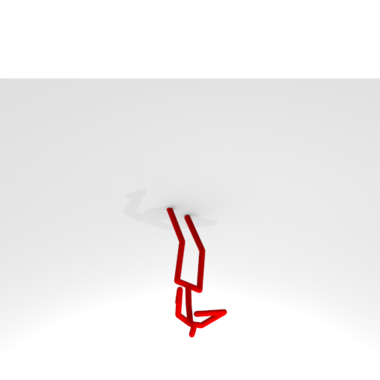}\\ 
\end{center} \vspace{-0.5cm}
\caption{\textbf{Effects of the canonicalization network $\Psi$.} 
Each column shows a 2D pose $Y$ input to the pose prediction network $\Phi$ (top) and the predicted 3D canonical shape $X=\Phi(Y)$ when $\Phi$ is trained with (middle) and without (bottom) the canonicalization network $\Psi$. Observe that training with $\Psi$ provides significantly more stable canonical shape predictions $X$ as the input pose rotates around the camera y-axis.
\label{f:democanonic}}
\end{figure}

A challenge with NR-SFM is the ambiguity in decomposing variations in the 3D shape of an object into viewpoint changes (rigid motions) and internal object deformations ~\cite{xiao2004dense}.
In this section, we propose a novel approach to directly encourage the reconstruction network $\Phi$ to be \emph{consistent} in the way reconstructions are performed.
This means that it must not be possible for the network to produce two different 3D reconstructions that differ only by a rigid motion, because such a difference should have been instead explained as a viewpoint change.

Formally, let $\mathcal{X}_0$ be the set of all reconstructions $X(\alpha;S)$ obtained by the network, where the parameters $(\alpha,\theta) = \Phi(Y,v)$ are obtained by considering all possible views $(Y,v)$ of the object.
If the network factorizes viewpoint and pose consistently, then there cannot be two different reconstructions $X, X' \in \mathcal{X
}_0$ related by a mere viewpoint change $X' = RX$.
This is formalized by the following definition:

\begin{definition}\label{d:transversal}
The set $\mathcal{X}_0\subset\mathbb{R}^{3\times K}$ has the \emph{transversal property} if, for any pair $X,X'\in\mathcal{X}_0$ of structures related by a rotation $X'= RX$, then $X=X'$.
\end{definition}

Transversality can also be interpreted as follows:
rotations partition the space of structures $\mathbb{R}^{3\times K}$ into equivalence classes.
We would like reconstructions to be unique within each equivalence class.
A set that has a unique or \emph{canonical} element for each equivalent class is also called a \emph{transversal}.
\Cref{d:transversal} captures this idea for the set of reconstructions $\mathcal{X}_0$.

For the purpose of learning, we propose to enforce transversality via the following characterizing property (proofs in the supplementary material):

\begin{lemma}\label{l:canonicalization}
The set $\mathcal{X}_0 \subset \mathbb{R}^{3\times K}$ has the transversal property if, and only if, there exists a \emph{canonicalization function} $\Psi : \mathbb{R}^{3\times K}\rightarrow\mathbb{R}^{3\times K}$ such that,
for all rotations $R\in SO(3)$ and structures $X \in \mathcal{X}_0$,
$
X = \Psi(RX).
$
\end{lemma}

Intuitively, this lemma states that, if $\mathcal{X}_0$ has the transversal property, then any rotation of its elements can be undone unambiguously.
Otherwise stated, we can construct a canonicalization function with range in the set of reconstructions $\mathcal{X}_0$ if, and only if, this set contains only canonical elements, i.e.~it has the transversal property~(\cref{d:transversal}).

For \methodname, the lemma is used to enforce a consistent decomposition in viewpoint and pose via the following loss:
\begin{equation}\label{e:secondary_net}
\ell_2(X,R;\Psi) = \frac{1}{K}\sum_{k=1}^{K} \| X_{:,k} - \Psi(R X)_{:,k} \|_\epsilon,
\end{equation}
where $R \in SO(3)$ is a randomly-sampled rotation, and $\Psi$ is a regressor
\emph{canonicalization}
network trained in parallel with the factorization network $\Phi$.

Regularizer $\ell_2$ (\cref{e:secondary_net}) is combined with the re-projection loss $\ell_1$ (\cref{e:basicloss}) as follows (\cref{f:overview}):
given an input sample $Y_n$, we first pass it through $\Phi(Y_n,v)$ to generate viewpoint and pose parameters $\theta_n$ and $\alpha_n$, which enter the re-projection loss $\ell_2$.
In addition, a random rotation $\hat{R}$ is applied to the generated structure $X_n = X(\alpha_{n};S)$, and $\hat{R}X_n$ is passed to the auxiliary canonicalization neural network $\Psi$.
$\Psi$ then undoes $\hat{R}$ by predicting shape coefficients $\hat \alpha_n$ that produce a shape $\hat X_n = X(\hat \alpha_{n};S)$ which should reconstruct the unrotated input shape $X_n$ as precisely as possible. This is enforced by passing $\hat X_n$ and $X_n$ to the loss $\ell_2$.
The two networks $\Phi$ and $\Psi$ are trained in parallel by minimizing $\ell_1 + \ell_2$, which encourages learning consistent viewpoint-pose factorization.
The effect of the loss is illustrated in~\cref{f:democanonic}.

\subsection{In-plane rotation invariance}\label{s:zequiv}

Rotation equivariance is another property of the factorization network that can be used to constrain learning.
Let $Y = \Pi R X$ be a view of the 3D structure $X$.
Rotating the camera around the optical axis has the effect of applying a rotation  $r_z\in SO(2)$ to the keypoints.
Hence, the two reconstructions $\Phi(Y,v)=(\alpha,\theta)$ and $\Phi(r_zY,v)=(\alpha',\theta')$ must yield the same 3D structure $\alpha=\alpha'$.
This is captured via a modified reprojection loss that exchanges $\alpha$ for $\alpha'$:
\begin{multline}\label{e:equivloss}
\ell_3(Y,v;\Phi,S)
\!=\!
\frac{1}{K}
\sum_{k=1}^K
v_k
\cdot
\| r_zY_k - M(\theta') (\alpha \otimes I_{3}) S_{:,k} \|_\epsilon
\end{multline}
This yields the combined loss $\ell_2 + \ell_3$ (the range of losses are comparable are combined with equal weight).

\newcommand{\hmgt}{H36Mgt}
\newcommand{\hmhg}{H36Mhg}
\newcommand{\abserr}{MPJPE}
\newcommand{\ablrepro}{\methodname-base}
\newcommand{\ablequiv}{\methodname-equiv}
\newcommand{\ablpsi}{\methodname}

\section{Experiments}\label{s:experiments}

\newcommand{\imw}{0.9cm}
\newcommand{\imh}{0.49cm}
\graphicspath{{./images/qual_new_v3/pascal3d_keypoints/}}

\newcommand{\methtext}[1]{\centering\scriptsize\hspace{0.01cm}\rotatebox{90}{\hspace{0.4cm}#1}\hspace{0.02cm}}
\newcommand{\cropbox}[1]{\noindent\adjustbox{trim={.12\width} {.2\height} {.12\width} {.18\height},clip,width=0.45\linewidth}{\includegraphics[width=\linewidth]{\rdrstyle/#1}}}

\newcommand{\pdrow}[1]{%
\begin{minipage}[c]{0.198\linewidth}\centering%
    \noindent\includegraphics[height=1.5cm]{#1_im.jpg}\\
    \methtext{\textbf{\color[RGB]{137,31,124}CMR}}%
    \cropbox{#1_kanazawa_000_090.png}%
    \cropbox{#1_kanazawa_000_135.png}\\
    \methtext{\textbf{\color[RGB]{181,38,25}Ours}}%
    \cropbox{#1_ours_000_090.png}%
    \cropbox{#1_ours_000_135.png}%
\end{minipage}}

\begin{figure*}[t]
\centering
\pdrow{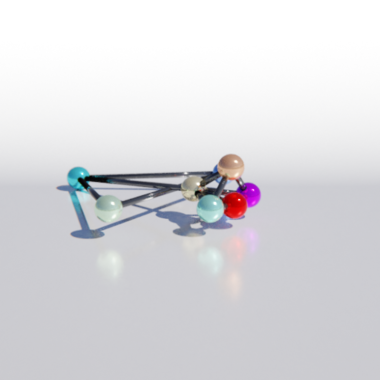}%
\pdrow{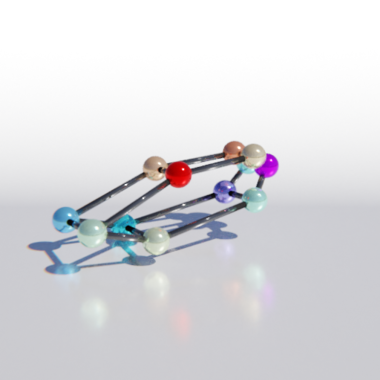}%
\pdrow{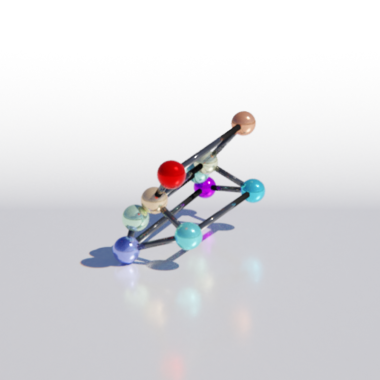}%
\pdrow{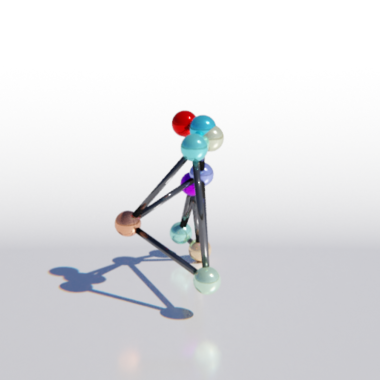}%
\pdrow{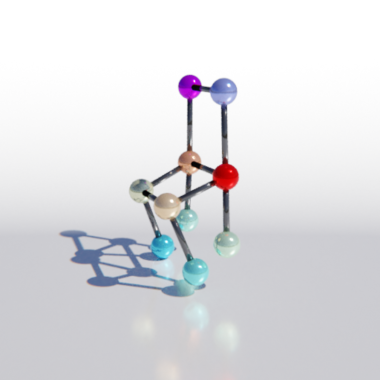}%
\vspace{\figmargin}
\caption{\textbf{Qualitative results on PASCAL3D+} comparing our method {\color[RGB]{181,38,25}\textbf{\ablpsi-HRNet}} (red) with {\color[RGB]{137,31,124}\textbf{CMR}} \cite{kanazawa2018learning} (violet). Each column contains the input monocular 2D keypoints (top) and lifting of the 2D keypoints into 3D by CMR (middle) and by our method (bottom) viewed from 2 different angles. \label{f:qual_p3d}}
\vspace{-0.3cm}
\end{figure*} 
\let\imw\undefined
\let\imh\undefined
\let\methtext\undefined
\let\cropbox\undefined
\let\pdrow\undefined

In this section, we compare our method against several strong baselines.
First, the employed benchmarks are described followed by quantitative and qualitative evaluations.

\subsection{Datasets}\label{s:datasets} 

We consider three diverse benchmarks containing images of objects with 2D keypoints annotations.
The datasets differ by keypoint density, object type, deformations, and intra-class variations.

\textbf{Synthetic Up3D (S-Up3D)}
We first validate \methodname~in a noiseless setting using a large synthetic 2D/3D dataset of dense human keypoints based on the Unite the People 3D (Up3D) dataset~\cite{lassner2017unite}.
For each Up3D image, the SMPL body shape and pose parameters are provided and are used to produce a mesh with 6890 vertices. Each of the 8515 meshes is randomly rotated into 30 different views and the orthographic projection of each vertex is recorded along with its visibility (computed using a ray tracer). The goal is then to recover the 3D shapes given the set of 2D keypoint renders. We maintain the same train/test split as in the Up3D dataset.

Similar to~\cite{lassner2017unite}, performance is evaluated on the 79 representative vertices of the SMPL model.
Although \methodname~can reconstruct the original set of 6890 SMPL model keypoints effortlessly, we evaluate on a subset of points due to a limited scalability of some of the baselines~\cite{torresani2008nonrigid,fragkiadaki2014grouping}.
For the same reason, we further randomly sampled the generated test poses to 15k images.
Performance is measured by averaging a 3D reconstruction error metric (see below) over all frames in the test set.

\textbf{PASCAL3D+}~\cite{xiang2014beyond}
Similar to~\cite{kanazawa2018learning,tulsiani2017learning}, we evaluate our method on the the PASCAL3D+ dataset which consists of PASCAL VOC and ImageNet images for 12 rigid object categories with a set of sparse keypoints annotated on each image (deformations still arise due to intra-class shape variations).
There are up to 10 CAD models available for each category, from which one is manually selected and aligned for each image, providing an estimate of the ground truth 3D keypoint locations.
To maintain consistency between the 2D and 3D keypoints, we use the 2D orthographic projections of the aligned CAD model keypoints as opposed to the per-image 2D keypoint annotations, and update the visibility indicators based on the CAD model annotations. 

\begin{table}
\centering
\setlength{\tabcolsep}{0.3cm}
\begin{tabular}{ l  r r r }
\toprule
Method                                 & \abserr         & Stress \\
\midrule
EM-SfM \cite{torresani2008nonrigid}    & 0.107          & 0.061 \\
GbNrSfM \cite{fragkiadaki2014grouping} & 0.093          & 0.062 \\
\textbf{\ablrepro}                     & 0.160          & 0.105 \\
\textbf{\ablequiv}                     & 0.154          & 0.102 \\
\textbf{\ablpsi}                       & \textbf{0.068} & \textbf{0.040} \\
\bottomrule
\end{tabular}
\vspace{\tabmargin}
\caption{\textbf{Results on the synthetic Up-3D (S-Up3D)} comparing our method (\ablpsi), NRSfM baselines \cite{torresani2008nonrigid,fragkiadaki2014grouping} and two variants of our method (\ablequiv, \ablrepro) which ablate effects of individual components of \ablpsi.}
\label{t:exp_up3d}
\vspace{-0.2cm}
\end{table}

\textbf{Human3.6M}~\cite{ionescu2014human36M} is perhaps the largest dataset of human poses annotated with 3D ground truth extracted using MoCap systems.
As in \cite{kudo2018unsupervised}, two variants of the dataset are used: the first contains ground-truth 2D keypoints during both train and test time and in the second, 2D keypoint locations are obtained by the Stacked Hourglass network of~\cite{toshev2014deeppose}.
We closely follow the evaluation protocol of~\cite{kudo2018unsupervised} and report absolute errors measured over 17 joints without any procrustes alignment.
We maintain the same train and test split as~\cite{kudo2018unsupervised}, and report an average over errors attained for each frame in a given MoCap sequence of an action type.

\begin{table}[t]
\centering
\setlength{\tabcolsep}{0.3cm}
\begin{tabular}{lrr}
\toprule
Method & \abserr & Stress \\
\midrule
GbNrSfM \cite{fragkiadaki2014grouping}  & 184.6 & 111.3 \\
EM-SfM \cite{torresani2008nonrigid} & 131.0 & 116.8 \\
\textbf{\ablrepro} & 53.5 & 46.8 \\
\textbf{\ablequiv} & 50.1 & 44.5 \\
\textbf{\ablpsi}   & \textbf{38.0} & \textbf{32.6} \\
\midrule
CMR \cite{kanazawa2018learning}$^\dagger$  & 74.4 & 53.7 \\
\textbf{\ablpsi} + HRNet$^\dagger$ & \textbf{57.5} & \textbf{41.4} \\
\bottomrule
\end{tabular}
\vspace{\tabmargin}
\caption{ \textbf{Average reconstruction error (\abserr) and $\ell_1$ stress over the 12 classes of Pascal3D} comparing our method \ablpsi~with two ablations of our approach (\ablequiv, \ablrepro) and the methods from \cite{fragkiadaki2014grouping,kanazawa2018learning,torresani2008nonrigid}. Approaches marked with $^\dagger$ predict 3D shape without knowledge of the ground-truth 2D keypoints at test time. }
\label{t:exp_p3d_allclass}
\vspace{-0.2cm}
\end{table}

\textbf{CUB-200-2011}~\cite{Wah2011cubbirds} consists of 11,788 images of 200 bird species.
Each image is annotated with 2D locations of 15 semantic keypoints and corresponding visibility indicators.
There are no ground truth 3D keypoints for this dataset so we only perform a qualitative evaluation.
We use the 2D annotations from~\cite{kanazawa2018learning}.

\subsection{Evaluation metrics}\label{s:metrics}

As common practice, the absolute mean per joint position error is reported:
$
\mathbf{\abserr}(X^*, X)
=
\sum_{k=1}^K \|X_k - X_k^*\| / K,
$
where $X_k \in \mathbb{R}^3$ is the predicted 3D location of the $k$-th keypoint and $X_k^*$ is its corresponding ground-truth 3D location (both in the 3D frame of the camera).

In order to evaluate \abserr~properly, two types of projection ambiguities have to be handled.
To deal with the \textbf{absolute depth} ambiguity, for Human3.6M we follow~\cite{kudo2018unsupervised} and normalize each pose by applying a translation that puts the skeleton root to the origin of the coordinate system.
For PASCAL3D+ and S-Up3D, the mean depth of predicted and ground truth point clouds is zero centered before evaluation.
The second, \textbf{depth flip} ambiguity, is resolved
as in~\cite{torresani2008nonrigid} by evaluating \abserr~twice for the original and depth-flipped point cloud, retaining the better of the two.

We also report the $\ell_1$
$
\mathbf{Stress}(X, X^*)
=
\sum_{i<j}
|~\|X_i-X_j\| - \|\hat X_i^*-X_j^*\|~|_1 / (K(K-1)).
$
This metric is invariant to camera pose and the absolute depth and z-flip ambiguities.

\subsection{Baselines}

\newcommand{\imw}{0.9cm}
\newcommand{\imh}{0.49cm}
\graphicspath{{./images/qual_new_v2/h36m_keypoints_json/}}

\newcommand{\methtext}[1]{\centering\scriptsize\hspace{0.01cm}\rotatebox{90}{\hspace{0.4cm}#1}\hspace{0.02cm}}
\newcommand{\methtextO}[1]{\centering\scriptsize\hspace{0.01cm}\rotatebox{90}{\hspace{0.6cm}#1}\hspace{0.02cm}}
\newcommand{\cropbox}[1]{\noindent\adjustbox{trim={.15\width} {.0\height} {.15\width} {.1\height},clip,width=0.45\linewidth}{\includegraphics[width=\linewidth]{\rdrstyle/#1}}}

\newcommand{\hmrow}[1]{%
\begin{minipage}[c]{0.198\linewidth}\centering%
    \noindent\includegraphics[height=2cm]{#1_im.jpeg}\\
    \methtext{\textbf{\color[RGB]{66,154,201}PoseGAN}}%
    \cropbox{#1_posegan_000_090.png}%
    \cropbox{#1_posegan_000_135.png}\\
    \methtextO{\textbf{\color[RGB]{181,38,25}Ours}}%
    \cropbox{#1_ours_000_090.png}%
    \cropbox{#1_ours_000_135.png}%
\end{minipage}}

\begin{figure*}[t]
\centering
\hmrow{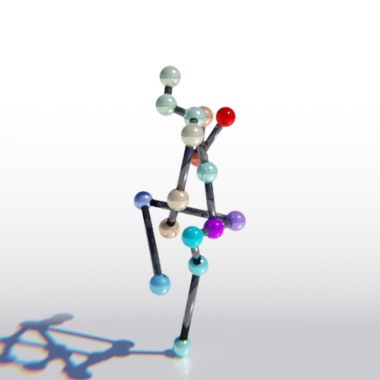}%
\hmrow{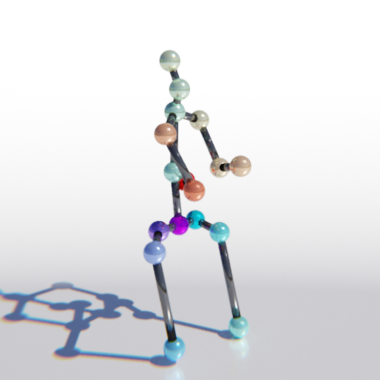}%
\hmrow{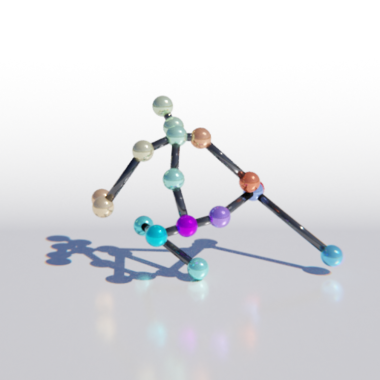}%
\hmrow{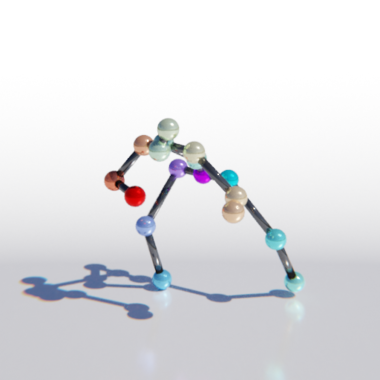}%
\hmrow{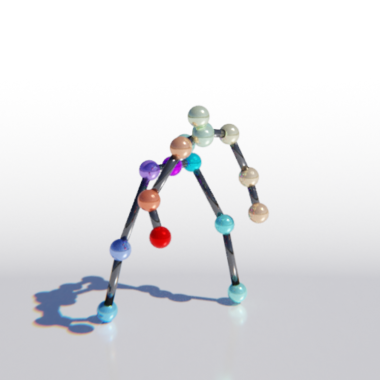}%
\vspace{\figmargin}
\caption{\textbf{3D poses on Human3.6M} predicted from monocular keypoints. Each column contains the input 2D keypoints (top) and a comparison between {\color[RGB]{66,154,201}\textbf{PoseGAN}} \cite{kudo2018unsupervised} (blue, middle), and our method {\color[RGB]{181,38,25}\textbf{\methodname}} (bottom, red) from two 3D viewpoints.
\label{f:qual_h36m}}
\vspace{-0.3cm}
\end{figure*} 
\let\imw\undefined
\let\imh\undefined
\let\methtext\undefined
\let\cropbox\undefined
\let\hmrow\undefined

\methodname~is compared to several strong baselines.
\textbf{EM-SfM}~\cite{torresani2008nonrigid} and \textbf{GbNrSfM}~\cite{fragkiadaki2014grouping} are NR-SFM methods with publicly available code.
Because, when using \cite{fragkiadaki2014grouping,torresani2008nonrigid}, it is difficult to make predictions on previously unseen data, we run the two methods directly on the test set and report results after convergence.
This gives the two baselines an advantage over our method.
On Human3.6M, out of several available methods, we compare with~\cite{kudo2018unsupervised} (\textbf{Pose-GAN}) which is a current \sota~approach for unsupervised 3D pose estimation that does not require any 3D, multiview or video annotations.
Unlike other weakly supervised methods~\cite{drover20183dpose,kocabas2019self}, Pose-GAN does not assume knowledge of the camera intrinsic parameters, hence it is the most comparable to our approach. To ensure fair comparison, we use their public evaluation code together with the provided keypoint detections of the Stacked Hourglass model. Pose-GAN was not tested on other datasets as the method cannot handle inputs with occluded keypoints.
On PASCAL3D+, our method is compared with Category-Specific Mesh Reconstruction (\textbf{CMR}) from \cite{kanazawa2018learning}.
CMR provides results for 2 categories out of the 12 of PASCAL3D+, but we trained models for all 12 categories using the public code.
Note that CMR additionally uses segmentation masks during training, hence has a higher level of supervision than our method.

\begin{table}[t]
\centering
\begin{tabular}{l r r r r }
\toprule
\multirow{ 2}{*}{Method} & 
\multicolumn{2}{c}{Ground truth pose} &  
\multicolumn{2}{c}{Predicted pose} \\
\cmidrule(lr){2-3} \cmidrule(lr){4-5}
& \abserr & Stress &  \abserr & Stress \\
\midrule				
Pose-GAN \cite{kudo2018unsupervised}   & 130.9          & 51.8          & 173.2   & -       \\
\textbf{\ablrepro}                     & 135.2          & 56.9          & 201.6 & 101.4 \\
\textbf{\ablequiv}                     & 128.2          & 53.0          & 190.4 & 93.9 \\
\textbf{\ablpsi}                       & \textbf{101.8} & \textbf{43.5} & \textbf{153.0} & \textbf{86.0} \\
\bottomrule
\end{tabular}
\vspace{\tabmargin}
\caption{\textbf{Results on Human3.6M} reporting average per joint position error (\abserr) and $\ell_1$ stress over the set of test actions (follows the evaluation protocol from \cite{kudo2018unsupervised}). We compare performance, when ground truth pose keypoints are available during test-time (2nd and 3rd column) and when the keypoints are predicted using the Stacked Hourglass network \cite{toshev2014deeppose} (4th and 5th column).}
\label{t:exp_h36M}
\vspace{-0.15cm}
\end{table}

The effects of individual components of our method are evaluated by ablation and recording the change in performance.
This generates three variants of our method:
(1) \textbf{\ablrepro} only optimizes the re-projection loss $\ell_1(\Phi)$ from~\cref{e:basicloss},
(2) \textbf{\ablequiv} replaces $\ell_1(\Phi)$ with the optimization of the z-invariant loss $\ell_3(\Phi)$ (\cref{s:zequiv}),
(3) \textbf{\ablpsi} extends \ablequiv~with the secondary canonicalization network $\Psi$ (\cref{s:canonicalization}).

\newcommand{\imw}{0.9cm}
\newcommand{\imh}{0.49cm}
\graphicspath{{./images/qual_new_v3/cub_birds_keypoints/}}

\newcommand{\methtext}[1]{\centering\scriptsize\hspace{0.01cm}\rotatebox{90}{\hspace{0.4cm}#1}\hspace{0.01cm}}
\newcommand{\cropbox}[1]{\noindent\adjustbox{trim={.11\width} {.07\height} {.11\width} {.18\height},clip,width=0.45\linewidth}{\includegraphics[width=\linewidth]{\rdrstyle/#1}}}

\newcommand{\birdrow}[1]{%
\begin{minipage}[c]{0.198\linewidth}\centering%
    \noindent\includegraphics[height=1.5cm]{#1_im.jpg}\\
    \methtext{\textbf{\color[RGB]{137,31,124}CMR}}%
    \cropbox{#1_kanazawa_000_000.png}%
    \cropbox{#1_kanazawa_000_090.png}%
    \methtext{\textbf{\color[RGB]{181,38,25}Ours}}%
    \cropbox{#1_ours_000_000.png}%
    \cropbox{#1_ours_000_090.png}%
\end{minipage}}

\begin{figure*}[t] 
\centering
\birdrow{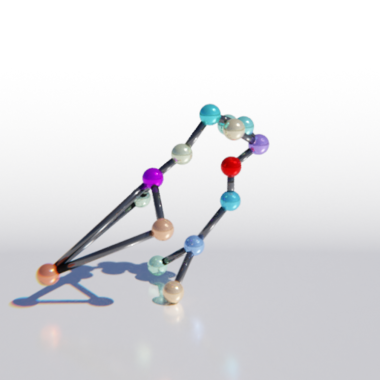}%
\birdrow{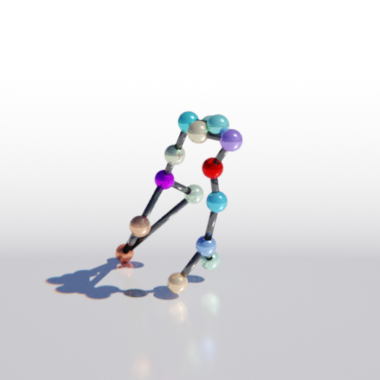}%
\birdrow{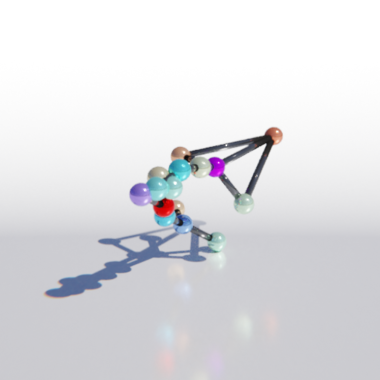}%
\birdrow{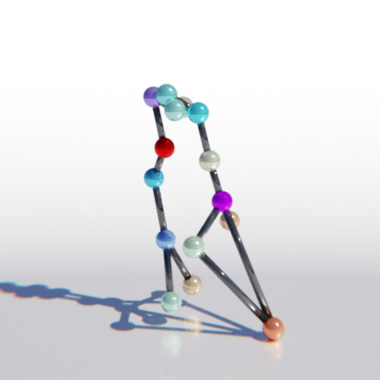}%
\birdrow{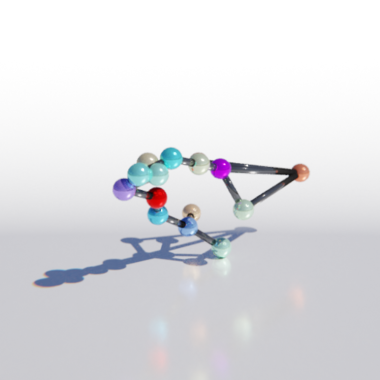}%
\vspace{\figmargin}
\caption{\textbf{Qualitative results on CUB-200-2011} comparing our method {\color[RGB]{181,38,25}\textbf{\ablpsi-HRNet}} (red) with {\color[RGB]{137,31,124}\textbf{CMR}} \cite{kanazawa2018learning} (violet). Each column contains the input monocular 2D keypoints (top), lifting of the 2D keypoints into 3D by CMR (middle) and by our method (bottom) from 2 different 3D viewpoints (the same view and a view offset by 90$^{\circ}$ along camera y-axis). \label{f:qual_birds} } 
\vspace{-0.3cm}
\label{fig:qual_birds_rebuttal}
\end{figure*}
\let\imw\undefined
\let\imh\undefined
\let\methtext\undefined
\let\cropbox\undefined
\let\birdrow\undefined

\subsection{Technical details} \label{s:details}

Networks $\Psi$ and $\Phi$ share the same core architecture and consist of 6 fully connected residual layers each with 1024/256/1024 neurons (please refer to the supplementary material for architecture details).
Residual skip connections were found important since they prevented networks from converging to a rigid average shape.

Keypoints $Y_n$ are first zero-centered before being passed to $\Psi$.
We further scale each set of centered 2D locations by the same scalar factor so their extent is roughly $[-1,1]$ on average in the axis of the highest variance.
The network is trained with a 
batched SGD
optimizer with momentum with an initial learning rate of 0.001, decaying 10 fold whenever the training objective plateaued. The batch size was set to 256.
The training losses $\ell_3(\Phi)$ and $\ell_2(\Psi)$ were weighted equally.

For Human3.6M, we did not model the translation $T$ of the camera as the centroid of the input 2D keypoints coincides with the centroid of the 3D shape (due to the lack of occluded keypoints).
For the other datasets, which contain occlusions, we estimate the camera translation as the difference vector between the mean of the input visible points and the re-projected visible 3D shape keypoints.

In order to adapt our method for the multiclass setting of PASCAL3D+, which has different sets of keypoints for each of the 12 object categories, we adjust the keypoint annotations as follows.
For each object category $\mathcal{C} \in \{1,\dots,12\}$ with a set of $K_\mathcal{C}$ keypoints $Y_n^\mathcal{C} \in \mathcal{R}^{2 \times K_\mathcal{C}}$ in an image $n$, we form a multiclass keypoint annotation
$
Y_n = \left[ \mathbf{0}, \dots, Y_n^\mathcal{C}, \dots, \mathbf{0} \right]
$
by assigning ${Y_n^\mathcal{C}}$ to the $\mathcal{C}$-th block of $Y_n$ and padding with zeros.
The visibility indicators $v_n$ are expanded in a similar fashion.
This avoids reconstructing each class separately, allowing our method to train only once for all classes.
This also tests the ability of the model to capture non-rigid deformations not only within, but also across object categories.
While this expanded version of keypoint annotations was also tested for GbNrSfM, for EM-SfM, we could not obtain satisfactory performance and reconstructed each class independently.
Similarly for CMR, 12 class-specific models were trained separately.

\subsection{Results}

\paragraph{Synthetic Up3D.} \Cref{t:exp_up3d} reports the results on the S-Up3D dataset.
Our method outperforms both EM-SfM and GbNrSfM, which validates our approach as a potential replacement for existing NR-SFM methods based on matrix factorization.
The table also shows that~\ablpsi~performs substantially better than~\ablrepro, highlighting the importance of the canonicalization network $\Psi$.

\paragraph{PASCAL3D+.}
For PASCAL3D+ we consider two types of methods.
Methods of the first type include GbNrSfM and EM-SfM and take as input 2D ground truth keypoint annotations on the PASCAL3D+ test set, reconstructing it directly.
The second type is CMR which uses ground truth annotations for training, but does not use keypoint annotations for evaluation on the test data.
In order to make our method comparable with CMR, we used as a detector the High Resolution Residual network (HRNet~\cite{sun2019deep}), training it on the 2D keypoint annotations from the PASCAL3D+ training set.
The trained HRNet is applied to the test set to extract the 2D keypoints $Y$ and these are lifted to 3D by applying \ablpsi ~(abbreviated as \textbf{\ablpsi+HRNet}).

The results are reported in~\cref{t:exp_p3d_allclass}.
\ablpsi~performs better than EM-SfM and GbNrSfM when ground truth keypoints are available during testing.
Our method also outperforms CMR by 16\%.
On several classes (motorbike, train), we obtain significantly better results due to the reliance of CMR on an initial off-the-shelf rigid SFM algorithm that fails to obtain satisfactory reconstructions.
This result is especially interesting since, unlike CMR, \ablpsi~is trained for all classes at once without ground truth segmentation masks.
\Cref{f:qual_p3d} contains qualitative evaluation.

\paragraph{Human3.6M.}
Results on the Human3.6M dataset are summarized in~\cref{t:exp_h36M}. \ablpsi~outperforms Pose-GAN for both ground truth and predicted keypoint annotations.
Again, \ablpsi~improves over baseline \ablrepro~by a significant margin.
Example reconstructions are in \Cref{f:qual_h36m}.

\paragraph{CUB-200-2011.}
Similar to PASCAL3D+, in order to make our method comparable with CMR, HRNet is trained on keypoints from the CUB-200-2011 train set and used to predict keypoints on unseen test images which are then input to \ablpsi. \Cref{f:qual_birds} compares qualitatively our reconstructions to CMR.
Our method is capable of modelling more flexible poses than CMR.
We hypothesize this is because of the reliance of CMR on an estimate of the camera matrices obtained using rigid SFM which limits the flexibility of the learned deformations.
On the other hand, CMR does not use a keypoint detector.
\section{Conclusions}\label{s:conclusions}

We have proposed a new approach to learn a model of a 3D object category from unconstrained monocular views with 2D keypoints annotations.
Compared to traditional solutions that cast this as NR-SFM and solve it via matrix factorization, our solution is based on learning a deep network that performs monocular 3D reconstruction and factorizes internal object deformations and viewpoint changes.
While this factorization is an ambiguous task, we have shown a novel approach that constrains the solution recovered by the learning algorithm to be as consistent as possible by means of an auxiliary canonicalization network. We have shown that this leads to considerably better performance, enough to outperform strong baselines on benchmarks that contain large non-rigid deformations within a category (Human3.6M, Up3D) and across categories (PASCAL3D+).
{\bibliographystyle{ieee_fullname}\bibliography{main}}
{\newpage\cleardoublepage\appendix\begin{strip}%
 \centering
 \Large
 \textbf{
    C3PO: Canonical 3D Pose Networks for Non-Rigid Structure From Motion \\ \vspace{0.3cm} \textit{Supplementary material}
 }
\vspace{0.5cm}
\end{strip}

The first part of the supplementary material contains discussions regarding the role of the camera translation in the formulation of SFM/NR-SFM (\cref{s:centering}), number of degrees of freedom (\cref{s:dof}) and a proof of \cref{l:canonicalization}.
Additional information about the architecture of the proposed deep networks is in \cref{s:supp_arch}.
\Cref{s:supp_robust} provides additional analysis of the robustness of \methodname ~ to the input noise. \Cref{s:supp_qual_more} presents additional qualitative results and \cref{s:supp_failure} discusses failure modes of our method.

\section{Theoretical analysis}

This section contains additional information regarding various theoretical aspects of the NR-SFM task.

\subsection{Centering}\label{s:centering}

This section summarizes well known results on data centering in orthographic SFM and NR-SFM.

\begin{lemma}
Equations $y_{nk} = \Pi R_n X_k + \Pi T_n$ hold true for all $n=1,\dots,N$ and $k=1,\dots,K$ if, and only if, equations $\bar y_{nk} = \Pi R_n \bar X_k$ hold true, where
$$
\bar y_{nk}  = y_{nk} - \frac{1}{K}\sum_{k=1}^K y_{nk},
\quad
\bar X_{k}  = X_{k} - \frac{1}{K}\sum_{k=1}^K X_{k}. 
$$
\end{lemma}
\begin{proof}
Average and remove the LHS and RHS of each equation from both sides.
\end{proof}

\begin{lemma}
Equation $y_{nk} = \Pi(R_n \sum_{d=1}^D \alpha_{nd} S_{dk} + T_n)$ holds true for all $n=1,\dots,N$ and $k=1,\dots,K$ if, and only if, equation $\bar y_{nk} = \Pi R_n \left(\sum_{d=1}^D \alpha_{nd} \bar S_{dk}\right)$ holds true, where
$$
\bar y_{nk}  = y_{nk} - \frac{1}{K}\sum_{k=1}^K y_{nk},
\quad
\bar S_{dk}  = S_{dk} - \frac{1}{K}\sum_{k=1}^K S_{dk}.
$$
\end{lemma}

\begin{proof}
Average and remove the LHS and RHS of each equation from both sides.
\end{proof}

\subsection{Degrees of freedom and ambiguities}\label{s:dof}

Seen as matrix factorization problems, SFM and NR-SFM have \emph{intrinsic ambiguities}; namely, no matter how many points and views are observed, there is always a space of equivalent solutions that satisfy all the observations.
Next, we discuss what are these ambiguities and under which conditions they are minimized.

\subsubsection{Structure from motion}

The SFM~\cref{e:sfm} contains $2NK$ constraints and $6N + 3K$ unknowns.
However, there is an unsolvable ambiguity: $MX = (MA^{-1}) (AX)$ means that, if $(M,X)$ is a solution, so $(MA^{-1}, AX)$ is another, for any invertible matrix $A \in \mathbb{R}^{3\times 3}$.
If $X$ is full rank and there are at least $N\geq 2$ views, we can show that this is the \emph{only} ambiguity, which has 9 degrees of freedom (DoF).
Thus finding a unique solution up to these residual 9 DoF requires $2NK \geq 6N + 3K - 9$.
For example, with $N=2$ views, we require $K \ge 3$ keypoints.
Furthermore, the 3D point configuration must not be degenerate, in the sense that $X$ must be full rank.

The ambiguity can be further reduced by considering the fact that the view matrices $M$ are not arbitrary; they are instead the first two rows of rotation matrices.
We can exploit this fact by setting $M_1 = I_{2\times 3}$ (which also standardize the rotation of the first camera), fixing 6 of the 9 DoF.
    
\newcommand{\imw}{0.9cm}
\newcommand{\imh}{2cm}
\newcommand{\imwin}{1.9cm}
\graphicspath{{./images/qual_new_v2/up3d_keypoints/}}

\newcommand{\cropbox}[1]{\noindent\adjustbox{trim={.01\width} {.02\height} {.01\width} {.02\height},clip}{\includegraphics[height=\imh]{\rdrstyle/#1}}}
\newcommand{\cropboxx}[1]{\noindent\adjustbox{trim={.3\width} {.3\height} {.3\width} {.3\height},clip}{\includegraphics[height=5cm]{#1}}}

\newcommand{\uprow}[1]{%
\begin{minipage}[c]{0.245\linewidth}\centering\footnotesize%
    \cropboxx{#1_image_000_000_kp.png}\\
    \cropbox{#1_ours_000_090.png}%
    \cropbox{#1_ours_000_135.png}%
\end{minipage}}

\begin{figure*}[t]
\centering
\uprow{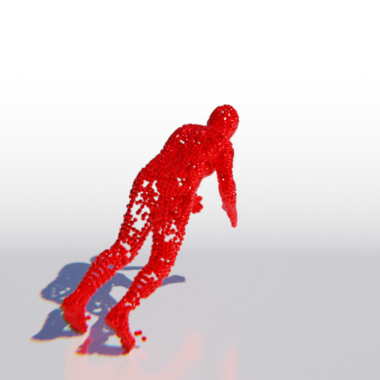}%
\uprow{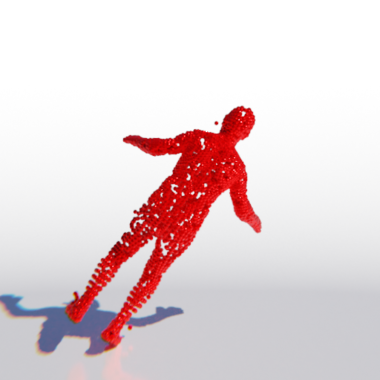}%
\uprow{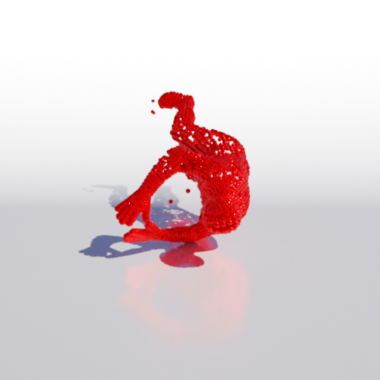}%
\uprow{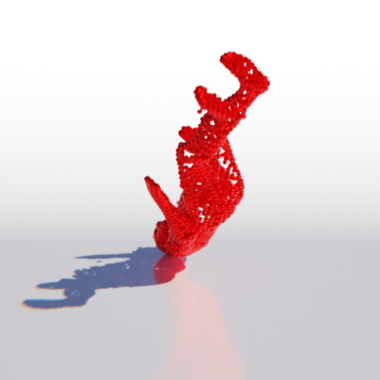}%
\vspace{\figmargin}
\caption{\textbf{Qualitative results on S-Up3D} showing input 2D keypoint annotations (top row) and monocular 3D reconstructions of all 6890 vertices of the SMPL model as predicted by \methodname ~ from two different viewpoints (bottom row).\label{f:qual_up3d}}
\vspace{-0.3cm}
\end{figure*} 
\let\imw\undefined
\let\imh\undefined
\let\methtext\undefined
\let\cropbox\undefined
\let\uprow\undefined

\subsubsection{Non-rigid structure from motion}

The NR-SFM equation contains $2NK$ constraints and $6N + ND + 3DK$ unknowns.
The intrinsic ambiguity has at least 9 DoF as in the SFM case. 
Hence, for a unique solution (up to the intrinsic ambiguity) we must have $2NK \geq 6N + ND + 3DK - 9$.
Compared to the SFM case, the number of unknowns grows with the number $N$ of views as $(6 + D)N$ instead of just $6N$, where $D$ is the dimension of the shape basis.
Since the number of constraints grows as $(2K)N$, we must have $K \geq 3 + D/2$ keypoints.

Note that once the shape basis $S$ is learned, it is possible to perform 3D reconstruction from a single view by solving~\eqref{e:nrsfm} for $N=1$; in this case there are $2K$ equations and $6 + D$ unknowns, which is once more solvable when $K \geq 3 + D/2$.

\subsection{Proof of \cref{l:canonicalization} }

\begin{figure}
    \centering
    \includegraphics[width=\linewidth]{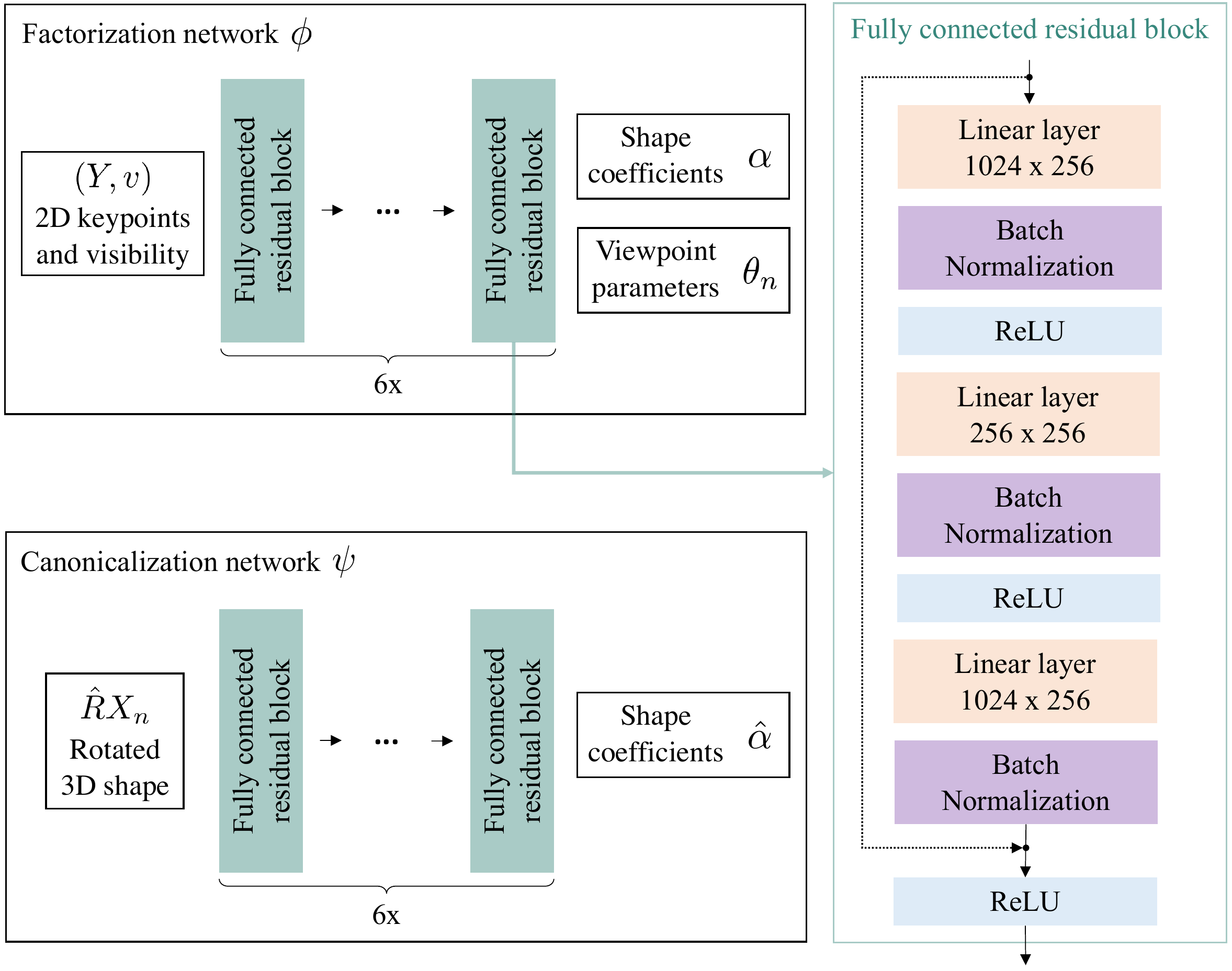}
    \caption{\textbf{The architecture of $\Psi$ and $\Phi$.} Both networks share the same trunk (6x fully connected residual layers) and differ in the type of their inputs and outputs.}
    \label{f:supp_net}
    \end{figure}

\begin{lemma}
The set $\mathcal{X}_0 \subset \mathbb{R}^{3\times K}$ has the transversal property if, and only if, there exists a \emph{canonicalization function} $\Psi : \mathbb{R}^{3\times K}\rightarrow\mathbb{R}^{3\times K}$ such that,
for all rotations $R\in SO(3)$ and structures $X \in \mathcal{X}_0$,
$
X = \Psi(RX).
$
\end{lemma}

\begin{proof}
Assume first that $\mathcal{X}_0$ has the transversal property.
Then the function $\Psi$ is obtained by sending each $RX$ for each $X \in\mathcal{X}_0$ back to $X$.
This definition is well posed: if $RX=\bar R \bar X$ where both $X,\bar X\in\mathcal{X}_0$, then $\bar X = (\bar R)^{-1}R X$ and, due to the transversal property, $X=\bar X$.

Assume now that the function $\Psi$ is given and let $X,X'\in\mathcal{X}_0$ such that $X' = RX$ and so $\Phi(X') = \Phi(RX)$.
However, by definition, $\Phi(RX)=X$ and $\Phi(X')=\Phi(IX')=X'$, so that $X=X'$.
\end{proof}

\section{Architecture of $\Psi$ and $\Phi$} \label{s:supp_arch}

\Cref{f:supp_net} contains a schema of the architecture of $\Psi$ and $\Phi$ (both share the same core architecture). It consists of 5 fully connected residual blocks with a kernel size of 1. Empirically, we have observed that using residual blocks, instead of the simpler variant with fully connected layers directly followed by batch normalization and no skip connections, prevents the network from predicting flattened shapes.

\section{Analysis of robustness} \label{s:supp_robust}

In order to test the robustness of \methodname~to the noise present in the input 2D keypoints, we devised the following experiment.

We generated several noisy versions of the Synthetic Up3D dataset by adding 2D Gaussian noise (with variance $\sigma$) to the 2D input and randomly occluded each 2D input point with probability $p_{OCC}$. Experiments were ran for different number of input of keypoints (79, 100, 500, 1000) and the evaluation was always conducted on the representative 79 vertices (\cref{s:datasets}) of S-Up3D-test.

The results of the experiment are depicted in \cref{f:supp_noise}. We have observed improved robustness to noise with higher numbers of used keypoints. At the same time, the performance without noise ($\sigma=0$, $p_{OCC}=0$) is slightly worse for the setup higher number of keypoints ($\geq$ 500 keypoints). We hypothesize that, when more keypoints are used, the performance deteriorates because the optimizer focuses less on minimizing the reprojection losses of the 79 keypoints that are used for the evaluation.

\section{Additional qualitative results} \label{s:supp_qual_more}

In this section we present additional qualitative results. \Cref{f:qual_up3d} contains monocular reconstructions of \methodname~trained on the full set of 6890 SMPL vertices of the S-Up3D dataset. Note that we were unable to run \cite{fragkiadaki2014grouping,torresani2004learning} on this dataset due to scalability issues of the two algorithms.

\section{\methodname~failure modes} \label{s:supp_failure}

\newcommand{\imw}{0.9cm}
\newcommand{\imh}{0.49cm}
\graphicspath{{./images/qual_new_v2/cub_birds_failure/}}
\newcommand{\methtext}[1]{\centering\scriptsize\hspace{0.01cm}\rotatebox{90}{\hspace{0.4cm}#1}\hspace{0.01cm}}
\newcommand{\cropbox}[1]{\noindent\adjustbox{trim={.12\width} {.07\height} {.12\width} {.18\height},clip,width=0.45\linewidth}{\includegraphics[width=\linewidth]{#1}}}
\newcommand{\birdrow}[1]{%
\begin{minipage}[c]{0.49\linewidth}\centering%
    \noindent\includegraphics[height=1.5cm]{#1_im.jpg}\\
    \cropbox{#1_ours_000_090.png}%
    \cropbox{#1_ours_000_135.png}%
\end{minipage}}

\makeatletter
\setlength{\@fptop}{0pt}
\makeatother

\begin{figure*}[t]
    \begin{minipage}[t]{0.49\linewidth}
        \vspace{0.4cm}
        \centering
        \birdrow{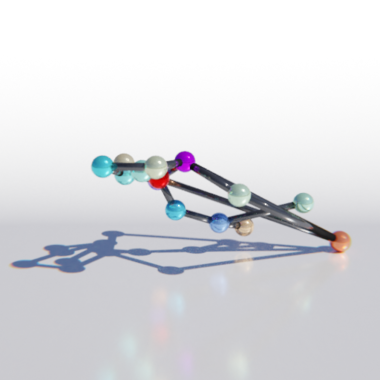}%
        \birdrow{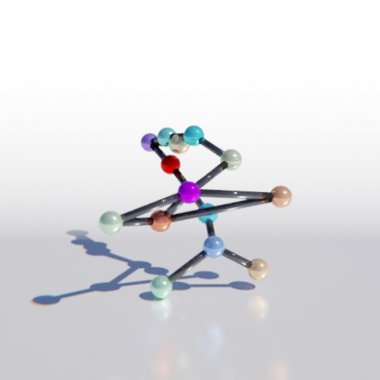}%
        \vspace{\figmargin}
        \caption{A qualitative example of 2D keypoints lifted by our method. Here, the reconstruction fails due to a failure of the HRNet keypoint detector. \label{fig:qual_birds_fail}} 
    \end{minipage}%
    \begin{minipage}[t]{0.49\linewidth}
        \newcommand{\imhh}{2.8cm}
        \newcommand{\titleoffs}{\hspace{0.6cm}}
        \centering \small
        \begin{minipage}[t]{0.49\linewidth}\centering
        \titleoffs\textit{79 keypoints}\\ \vspace{0.1cm}
        \includegraphics[height=\imhh]{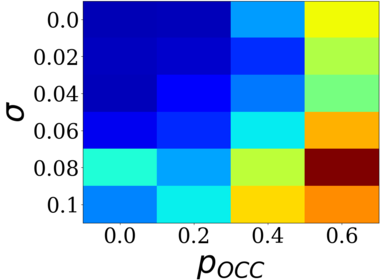}
        \end{minipage}%
        \begin{minipage}[t]{0.49\linewidth}\centering
        \titleoffs\textit{100 keypoints}\\ \vspace{0.1cm}
        \includegraphics[height=\imhh]{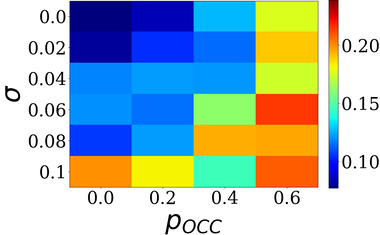}
        \end{minipage}\\ \vspace{0.3cm}
        \begin{minipage}[t]{0.49\linewidth}\centering
        \titleoffs\textit{500 keypoints}\\ \vspace{0.1cm}
        \includegraphics[height=\imhh]{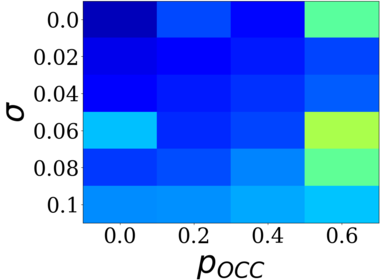}
        \end{minipage}%
        \begin{minipage}[t]{0.49\linewidth}\centering
        \titleoffs\textit{1000 keypoints}\\ \vspace{0.1cm}
        \includegraphics[height=\imhh]{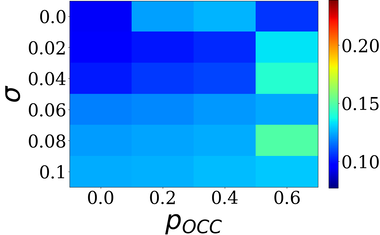}
        \end{minipage}
        \vspace{0.05cm}
        \caption{MPJPE on Up3D of \methodname~depending on various levels of Gaussian noise added to 2D inputs ($\sigma$-vertical axis) and the probability of occluding an input 2D point ($p_{OCC}$-horizontal axis) for different numbers of training keypoints (left to right, top to bottom: 79, 100, 500, 1000).}
        \vspace{-0.4cm}\label{f:supp_noise}
    \end{minipage} 
    \vspace{-0.3cm}
\end{figure*}
\let\imw\undefined
\let\imh\undefined
\let\methtext\undefined
\let\cropbox\undefined
\let\birdrow\undefined
\null

The main sources of failures of our method are:
(1) Failures of the 2D keypoint detector \cite{sun2019deep};
(2) Reconstructing ``outlier'' test 2D poses not seen in training (mainly on Human3.6m);
(3) Reconstructing strongly ambiguous 2D poses (in a frontal image of a sitting human, the knee angle cannot be recovered uniquely). The failure mode (1) is depicted in \cref{fig:qual_birds_fail}.
}
\end{document}